\documentclass[12pt]{article}

\usepackage{amssymb}
\usepackage{amsmath,amsthm}
\usepackage[latin1]{inputenc}
\usepackage{hyperref}
\usepackage{cite}
\usepackage{enumerate}

\hypersetup{colorlinks=true}

\hypersetup{colorlinks=true, linkcolor=blue, citecolor=blue,urlcolor=blue}

 \newtheorem{remark}{Remark}

 \newtheorem{theorem}[remark]{Theorem}

\title{MOCICE-BCubed F$_1$: A New Evaluation Measure for Biclustering Algorithms\footnote{This manuscript has been submitted to \emph{Pattern Recognition Letters}.}}

\author{Henry Rosales-M\'endez$^1$, Yunior Ram\'irez-Cruz$^2$\\{\small $^1$ Computer Science Department,}\\{\small Universidad de Oriente, Patricio Lumumba s/n, Santiago de Cuba, 90500, Cuba.}\\{\small (\emph{Currently with ICTexcelsus Ltd., Av. Rep. de El Salvador N36-161, Quito, 170505, Ecuador})}\\{\small $^2$ Departament d'Enginyeria Inform\`atica i Matem\`atiques,}\\
{\small Universitat Rovira i Virgili,}  {\small Av. Pa\"{\i}sos
Catalans 26, Tarragona, 43007, Spain.} \\{\small hrosmendez\@@gmail.com, yunior.ramirez\@@urv.cat}}

\begin{document}
\maketitle
\begin{abstract}
The validation of biclustering algorithms remains a challenging task, even though a number of measures have been proposed for evaluating the quality of these algorithms. Although no criterion is universally accepted as the overall best, a number of meta-evaluation conditions to be satisfied by biclustering algorithms have been enunciated. In this work, we present MOCICE-BCubed F$_1$, a new external measure for evaluating biclusterings, in the scenario where gold standard annotations are available for both the object clusters and the associated feature subspaces. Our proposal relies on the so-called micro-objects transformation and satisfies the most comprehensive set of meta-evaluation conditions so far enunciated for biclusterings. Additionally, the proposed measure adequately handles the occurrence of overlapping in both the object and feature spaces. Moreover, when used for evaluating traditional clusterings, which are viewed as a particular case of biclustering, the proposed measure also satisfies the most comprehensive set of meta-evaluation conditions so far enunciated for this task.

\vspace{0.3cm}
\end{abstract}

\section{Introduction}
\label{sec:intro}

The aim of clustering algorithms (also referred to as \emph{unsupervised classification algorithms}) is to structure a collection of objects into a set of groups, or \emph{clusters}, aiming to place dissimilar objects in different clusters, and similar objects in the same cluster. The solutions to a large number of real world problems may be modelled by clusterings. Also, clustering is commonly used as an auxiliary task in many fields, e.g. wireless sensor networks \cite{wireless}, speech recognition \cite{speech}, stochastic optimization \cite{stochastic_opt}, data compression \cite{datacompression}, document organization, etc.

In traditional clustering, the feature  space on which objects are represented is determined off-line, and the representation of every object for performing the clustering method is determined on this feature space. In the last decades, a generalization of the traditional clustering task, called \emph{biclustering}\footnote{A wide variety of terms have been used to refer to this task. While it is called \emph{biclustering} in \cite{Turner05}, it is also referred to as \emph{co-clustering} \cite{Cho04}, \emph{subspace clustering} \cite{patrikainen06}, \emph{projection / projected / projective clustering} \cite{Aggarwal00},~etc.}, has emerged. The underlying idea of biclustering is to collaboratively find adequate subspaces of features, in terms of which meaningful, high quality clusters may be discovered. A variety of biclustering algorithms have been proposed. We can find surveys about this topic in \cite{kriegel09,parsons04,muller09,moise09,yip03}. Biclustering has been applied in a wide range of problems, such as image segmentation\cite{yang08}, face clustering\cite{ho03}, image compression\cite{hong06}, genome expression data \cite{cha14}, etc.

Cluster validation is the field of study dealing with the methodologies aiming to asses the quality of the results of a clustering algorithm, which we refer to as \emph{candidate clustering}. The quality of a candidate clustering is assessed via one or several evaluation measures, which are expected to yield optimum scores for high quality candidate clusterings and far-from-optimum scores for poor candidate clusterings, as well as comparable scores for two or several comparable candidate clusterings. Validation criteria are divided into internal, external or relative. Relative validation measures choose the best results of multiple runs of a clustering algorithm with different parameters, whether these results have been obtained by means of an internal or external measure. Internal validation measures assess the quality of a candidate clustering by analyzing exclusively the group structure and/or the object-to-object, object-to-cluster and cluster-to-cluster relations observed in it, whereas external validation measures compare the candidate clustering to an ideal clustering, also called \emph{gold standard}. The gold standard is assumed to describe the correct clustering, i.e. the one that best fits the real world structure of the collection, and is usually the result of a manual annotation process conducted by one, or (desirably) several, human specialists. In the context of external evaluation measures, it is common to use the term \emph{cluster} only to refer to the clusters in the candidate clustering, whereas the clusters in the gold standard are called \emph{classes}, \emph{categories}, or \emph{hidden clusters}. For uniformity, throughout this work we will use the term \emph{classes} for referring to the clusters of the gold standard. 

The large number of evaluation measures proposed has brought up the need of developing \emph{meta-evaluation} criteria, which intend to assess the suitability of a given evaluation measure, or to compare two measures. Usually, these criteria are expressed as sets of conditions to be satisfied by ``good" evaluation measures. No set of conditions enjoys universal acceptation. Here, when treating measures for traditional clustering, we use the set of four conditions proposed by Amig\'o et al. in \cite{Amigo} for traditional clustering, along with an additional condition proposed by the authors of this work in \cite{cice,IDApaper} for the overlapping clustering scenario, as the basis for meta-evaluation. We do so because the conditions proposed in \cite{Amigo} were shown to subsume the previously existing conditions. For an analogous reason, when treating measures for biclustering, we additionally use the set of conditions proposed by Patrikainen and Meila in \cite{patrikainen06}.

Several studies have been conducted on external cluster validation in traditional clustering \cite{Amigo,Bagga,cice,Meila,Dom,Rosenberg,ramirez,goldberg,amigo2,Halkidi}. Although measures defined for this purpose may be used to partially evaluate biclusterings from the object space perspective, they are unable to take into account the quality of the feature space clustering. According to Patrikainen and Meila \cite{{patrikainen06}}, three different approaches have been followed in biclustering validation. On one hand, a number of authors have evaluated biclusterings from the object space perspective only, overlooking information about the feature space \cite{yip03,Aggarwal00,procopiuc02,Domeniconi04,Prelic06}. On the other hand, other authors only take into account the feature subspace perspective \cite{Patrikainen04}. Finally, a third approach consists on evaluating the quality from each perspective separately and merging the partial scores into one final score \cite{Cho04}. In every case, a measure that only takes into account the object (feature) space yields the same value for any biclustering whose object (feature) clusters are fixed, regardless the clustering on the feature (object) space. To overcome this problem, new measures have been proposed which deal with both perspectives in a joint manner \cite{gunnemann11}. 

Traditional clustering may be viewed as a particular case of biclustering, where a fixed feature subspace is associated to every clustering. In light of this consideration, it is reasonable to expect that biclustering evaluation measures, when applied in this scenario, satisfy traditional clustering meta-evaluation conditions. However, as we will show later, that is not always the case. Motivated by this problem, in this paper we present a new measure for biclustering evaluation, MOCICE-BCubed F$_1$, which builds on the measure CICE-BCubed F$_1$, known to satisfy the most comprehensive set of meta-evaluation conditions on traditional clustering. The new measure correctly adapts to the biclustering scenario by applying the so-called micro-objects transformation, and it satisfies the most comprehensive set of biclustering meta-evaluation conditions, while also inheriting the compliance to all traditional clustering meta-evaluation conditions.

The remainder of this paper is organized as follows. In Section~\ref{sec:related_work}, we briefly review previous work in biclustering algorithm evaluation, focusing on the most comprehensive set of meta-evaluation conditions for traditional clustering and biclustering, existing micro-object-based external evaluation measures and the fact that these measures fail to satisfy several meta-evaluation conditions when used for evaluating traditional clusterings. In Section~\ref{sec:ourProposal}, we describe the new proposed measure and prove its compliance to meta-evaluation conditions. Finally, we present our conclusions in Section~\ref{sec:conclusion}.

\section{Background and previous work}
\label{sec:related_work}

Given the pair $(O,F)$, where $O=\{o_1,o_2,\ldots,o_n\}$ is usually viewed as a set of objects and $F=\{f_1,f_2,\ldots,f_m\}$ is usually viewed as a set of features, a traditional clustering of $(O,F)$ is a set ${\cal G}=\{G_1,G_2,\ldots,G_t\}$, where $G_i \subseteq O$ for every $i \in \{1,\ldots,t\}$, whereas a biclustering of $(O,F)$ is a set $\ddot{\cal G}= \{\ddot{G}_1, \ddot{G}_2, \ldots, \ddot{G}_t\}$, where $\ddot{G}_i=(\bar{G}_i,\mathring{G}_i)$, $\bar{G}_i \subseteq O$ and $\mathring{G}_i \subseteq F$ for every $i \in \{1,\ldots,t\}$. Traditional clusterings may be considered as a particular case of biclusterings, where $\mathring{G}_i=\mathring{G}_j$ for every $i,j \in \{1,\ldots,t\}$. In particular, we can make $\mathring{G}_i=F$ for every $i \in \{1,\ldots,t\}$.

A biclustering $\ddot{\cal G}$ needs not satisfy $\displaystyle{\cup_{\ddot{G} \in \ddot{\cal G}}}\bar{G} = O$ nor $\displaystyle{\cup_{\ddot{G} \in \ddot{\cal G}}}\mathring{G} = F$. Moreover, for two biclusters $\ddot{G},\ddot{G}' \in \ddot{\cal G}$, the conditions $\bar{G} \cap \bar{G}'=\emptyset$ and $\mathring{G} \cap \mathring{G}'=\emptyset$ are not enforced neither, \emph{i.e.} overlapping is allowed on both the object space and the feature space.

Formally, an evaluation measure for traditional clusterings is a function of the form $$f: \rho(\rho(O)) \times \rho(\rho(O)) \longrightarrow \mathbb{R},$$ where $\rho(O)$ is the power set of $O$. Such a function takes a candidate clustering and a gold standard as arguments, and yields a score that indicates how good the candidate clustering is according to the gold standard. Higher scores are commonly interpreted as better, \emph{i.e.} the measure is assumed to assess the similarity between the candidate clustering and the gold standard, but that is not a mandatory behaviour, as a measure may also assess the dissimilarity between the candidate clustering and the gold standard. In an analogous manner, an evaluation measure for biclusterings is a function of the form $$f: \rho(\rho(O) \times \rho(F)) \times \rho(\rho(O) \times \rho(F)) \longrightarrow \mathbb{R}.$$

Several authors have proposed sets of meta-evaluation conditions for traditional clusterings \cite{Meila,Dom,Rosenberg}. A set of four conditions is proposed in \cite{Amigo} which subsumes those previously existing. An additional condition was proposed in \cite{cice,IDApaper} to account for special situations arising in overlapping clusterings. These conditions are enunciated as follows:

\begin{enumerate}[{A}.1-]
\item \textit{Homogeneity} \cite{Amigo}: Let ${\cal C}$ be a gold standard and let ${\cal G}_1$ be a clustering where one cluster $G_k$ contains objects belonging to two classes $C_i,C_j \in {\cal C}$. Let ${\cal G}_2$ be a clustering identical to ${\cal G}_1$, except for the fact that instead of the cluster $G_k$, it contains two clusters $G'_{k_1}$ and $G'_{k_2}$, one of them containing only objects belonging to $C_i$ and the other containing only objects belonging to $C_j$. An evaluation measure that satisfies the homogeneity condition should score ${\cal G}_1$ worse than ${\cal G}_2$.
\item \textit{Completeness} \cite{Amigo}:  Let ${\cal C}$ be a gold standard and let ${\cal G}_1$ be a clustering where two clusters $G_1$ and $G_2$ contain only objects belonging to one class $C_k \in {\cal C}$. Let ${\cal G}_2$ be a clustering identical to ${\cal G}_1$, except for the fact that instead of the clusters $G_1$ and $G_2$, it contains the cluster $G_{1,2}=G_1 \cup G_2$. An evaluation measure that satisfies the completeness condition should score ${\cal G}_1$ worse than ${\cal G}_2$.
\item \textit{Rag Bag} \cite{Amigo}: Let ${\cal C}$ be a gold standard. Let ${\cal G}_1$ be a clustering where one cluster $G_{clean}$ contains $n$ objects belonging to one class $C_i \in {\cal C}$ plus one object belonging to a different class $C_j \in {\cal C}-\{C_i\}$ and one cluster $G_{noise}$ contains $n$ objects belonging to $n$ different classes. Let ${\cal G}_2$ be a clustering identical to ${\cal G}_1$, except for the fact that the object in $G_{clean}$ that does not belong to the same class as all other objects is placed instead in $G_{noise}$. An evaluation measure that satisfies the rag bag condition should score ${\cal G}_1$ worse than ${\cal G}_2$.
\item \textit{Clusters size versus quantity} \cite{Amigo}: Let ${\cal C}$ be a gold standard. Let ${\cal G}$ be a clustering where one cluster $G_{large}$ contains $r+1$ objects belonging to one class $C_i \in {\cal C}$ and $r$ clusters $G_1$, $G_2$, \ldots , $G_r$, contain each one two objects belonging to the same class. Let ${\cal G}_1$ be a clustering identical to ${\cal G}$, except for the fact that instead of the two-object clusters $G_1$, $G_2$, \ldots , $G_r$, it contains $2r$ singleton clusters containing the corresponding objects. Let ${\cal G}_2$ be a clustering identical to ${\cal G}$, except for the fact that instead of the cluster $G_{large}$, it contains one cluster of size $r$ and one cluster of size 1. An evaluation measure that satisfies the clusters size versus quantity condition should score ${\cal G}_1$ worse than ${\cal G}_2$.
\item \textit{Perfect match} \cite{cice,IDApaper}: An evaluation measure must yield the optimum score for a candidate clustering if and only if it is identical to the gold standard.
\end{enumerate}

Regarding biclustering algorithms, a set of five conditions are presented in \cite{patrikainen06} which describe what is considered as a good behavior, so good evaluation measures are expected to reward algorithms that exhibit such behavior. These conditions are enunciated as follows:

\begin{enumerate}[{B}.1-]
\item \textit{Penalty for non-intersection area}: Let $\ddot{\cal G}$ be a biclustering and let $\ddot{\cal G}'$ be a biclustering identical to $\ddot{\cal G}$, except for the fact that one (or more) non-clustered object $o \in O-\left(\cup_{\ddot{G} \in \ddot{\cal G}}\bar{G} \displaystyle\bigcup \cup_{\ddot{C} \in \ddot{\cal C}}\bar{C}\right)$, is added to $\ddot{\cal G}'$, either as a new singleton cluster, or as part of an existing cluster. An evaluation measure that satisfies the penalty for non-intersection area condition should score $\ddot{\cal G}'$ worse than $\ddot{\cal G}$. 

\item \textit{Background independence}: The score yielded for a pair of biclusterings must not depend on non-clustered objects. Let $\ddot{\cal C}_X$ and $\ddot{\cal G}_X$ be a gold standard and a candidate biclustering, respectively, on a collection $X$. Let $X'$ be a collection such that $X \subset X'$ and let $\ddot{\cal C}_{X'}$ and $\ddot{\cal G}_{X'}$ represent $\ddot{\cal C}_{X}$ and $\ddot{\cal G}_{X}$ on $X'$. An evaluation measure that satisfies the background independence condition should yield the same score for $(\ddot{\cal G}_X,\ddot{\cal C}_X)$ and $(\ddot{\cal G}_{X'},\ddot{\cal C}_{X'})$.

\item \textit{Scale invariance}: For a positive integer $k$ and a biclustering $\ddot{\cal G}$, a $k$-scaled biclustering $\ddot{\cal G}'$ of $\ddot{\cal G}$ is a biclustering where, for every $\ddot{G}' \in \ddot{\cal G}'$, $\bar{G}'$ is the disjoint union of $k$ copies of $\bar{G}$ and $\mathring{G}'$ is the disjoint union of $k$ copies of $\mathring{G}$. Let $\ddot{\cal G}$ and $\ddot{\cal C}$ be a candidate biclustering and a gold standard, respectively, and let $\ddot{\cal G}'$ and $\ddot{\cal C}'$ be $k$-scaled biclusterings of $\ddot{\cal G}$ and $\ddot{\cal C}$. An evaluation measure that satisfies the scale invariance condition should yield the same score for $(\ddot{\cal G}',\ddot{\cal C}')$ and $(\ddot{\cal G},\ddot{\cal C})$. 

\item \textit{Copy invariance}: Let $\ddot{\cal C}$ be a gold standard and let $\ddot{\cal G}$ be a candidate biclustering. For a positive integer $k$, let $\ddot{\cal G}'$ be the disjoint union of $k$ copies of $\ddot{\cal G}$, and let $\ddot{\cal C}'$ be the disjoint union of $k$ copies of $\ddot{\cal C}$. An evaluation measure that satisfies the copy invariance condition should yield the same score for $(\ddot{\cal G}',\ddot{\cal C}')$ and $(\ddot{\cal G},\ddot{\cal C})$.

\item \textit{Multiple cluster coverage penalty}: Let $\ddot{\cal G}=\{\ddot{G}\}$ be a singleton biclustering. Let $\ddot{\cal C}=\{\ddot{C}_1,\ddot{C}_2,\ldots,\ddot{C}_t\}$ be a gold standard  such that $\bar{G}=\cup_{\ddot{C} \in \ddot{\cal C}}\bar{C}$ and $\mathring{G}=\cup_{\ddot{C} \in \ddot{\cal C}}\mathring{C}$. An evaluation measure that satisfies the multiple cluster coverage penalty condition should not yield the optimum score for $(\ddot{\cal G},\ddot{\cal C})$.
\end{enumerate}

We now discuss several external evaluation strategies proposed for the  biclustering scenario. We describe the micro-objects transformation and the measures that have been adapted to this approach, namely CE, RNIA, Rand's index, VI and E4SC.

Patrikainen and Meila \cite{patrikainen06} propose to transform the candidate biclustering and the gold standard into traditional clusterings in order to apply existing evaluation measures for the latter. They do so by transforming $(O,F)$ into the new space $(O \times F,\emptyset)$, where $O \times F$ is composed by pairs of the form $(o,f)$, $o \in O$, $f \in F$, which they call \emph{micro-objects}. Thus, a biclustering $\ddot{\cal G}=\{(\bar{G}_1,\mathring{G}_1),(\bar{G}_2,\mathring{G}_2),\ldots,(\bar{G}_t,\mathring{G}_t)\}$ is transformed into a clustering $\widetilde{\cal G}=\{\widetilde{G}_1,\widetilde{G}_2,\ldots,\widetilde{G}_t\}$ where $\widetilde{G}_i=\bar{G}_i \times \mathring{G}_i$ for every $i \in \{1,\ldots,t\}$.

Patrikainen and Meila introduce their proposals on non-overlapping biclusterings, which they propose to evaluate by applying the micro-objects transformation in combination with the measures \emph{Clustering Error} (\emph{CE}), \emph{Relative Non-intersecting Area} (\emph{RNIA}), \emph{Rand's index} and \emph{Variation of information} (\emph{VI}). CE determines the best matching between the candidate clustering and the gold standard, and computes the total number of objects shared by every class-cluster pair, according to this matching, which is denoted by $D_{max}$. CE is  defined as

\begin{equation}
\label{eq:CE}
CE({\cal G},{\cal C})=\frac{|U|-D_{max}}{|U|}
\end{equation}

\noindent
where $U=\left(\displaystyle{\cup_{G \in {\cal G}}}G\right) \bigcup \left( \displaystyle{\cup_{C \in {\cal C}}}C\right)$. Now, let $I = \left(\displaystyle{\cup_{G \in {\cal G}}}G\right) \bigcap \left(\displaystyle{\cup_{C \in {\cal C}}}C\right)$. RNIA is defined as 

\begin{equation}
\label{eq:RNIA}
RNIA({\cal G},{\cal C})=\frac{|U|-|I|}{|U|}
\end{equation}

The traditional Rand's index assumes the candidate clustering and the gold standard to be partitions of the object universe and is defined as

\begin{equation}
Rand({\cal G},{\cal C})=\frac{N_{00}+N_{11}}{N}\label{RandInd}
\end{equation}

\noindent
where $N_{01}$ is the number of object pairs that co-occur in a cluster of ${\cal G}$ and co-occur in a class of ${\cal C}$, $N_{00}$ is the number of object pairs that do not co-occur in a cluster of ${\cal G}$ and do not co-occur in a class of ${\cal C}$, and $N$ is the total number of object pairs. Patrikainen and Meila count $N$ on the universe $U=\left(\cup_{G \in {\cal G}}G\right) \cup \left(\cup_{C \in {\cal C}}C\right)$ and, to make the candidate clustering and the gold standard be partitions of $U$, they add as many singleton clusters as necessary.

Finally, VI is based on information theory and assesses the amount of information gained and lost when transforming the candidate clustering into the gold standard, as follows:

\begin{equation}
VI({\cal G},{\cal C})=\frac{1}{|U|}\displaystyle\sum_{i=1}^{t_1}\sum_{j=1}^{t_2}|G_i \cap C_j|\log\frac{|G_i|\cdot|C_j|}{|G_i \cap C_j|^2}\label{VI}
\end{equation}

\noindent
where $t_1=|{\cal G}|$ and $t_2=|{\cal C}|$. In a manner analogous as for Rand's index, they transform the candidate clustering and the gold standard into partitions of $U$ by adding as many singleton clusters as necessary.

Also following this approach, G\"unnemann et al. \cite{gunnemann11} propose to combine the micro-objects transformation and a variant of the F$_1$ measure called E4SC. They compute the macro-averaged F$_1$ of the candidate clustering with respect to the gold standard, and \emph{vice versa}, and compute the F$_1$ measure of both scores, as shown in Eq.~\ref{eq:E4SC}.

\begin{equation}
\label{eq:pr_rc}
P(A,B) = R(B,A) = \frac{|A \cap B|}{|A|}
\end{equation}

\begin{equation}
\label{eq:f1}
F_1(G,C) = \frac{2\cdot~P(G,C)\cdot~R(G,C)}{P(G,C)+R(G,C)}
\end{equation}

\begin{equation}
\label{eq:macrof1}
macroF_{1}({\cal G},{\cal C}) = \frac{1}{|{\cal G}|}\sum_{G \in {\cal G}}{\max_{C \in {\cal C}}{\{F_1(G,C)\}}}
\end{equation}

\begin{equation}
\label{eq:E4SC}
E4SC({\cal G},{\cal C}) = \frac{2\cdot~macroF_{1}({\cal G},{\cal C})\cdot macroF_{1}({\cal C},{\cal G}) }{macroF_{1}({\cal G},{\cal C})+macroF_{1}({\cal C},{\cal G})}
\end{equation}

Since the micro-objects transformation turns the evaluation of biclusterings into that of traditional clusterings, it is reasonable to expect that the measures applied on the micro-objects universe satisfy a wide range of meta-evaluation conditions for traditional clustering. However, the measures discussed so far fail to satisfy some of these conditions. Amig\'o et al. \cite{Amigo} show that Rand's index and VI do not satisfy the Rag Bag (A.3) condition, whereas Rand's index additionally fails to satisfy the Clusters size versus quantity (A.4) condition. Moreover, as we will show, CE, RNIA and E4SC also fail to satisfy some of these conditions. Consider the candidate biclusterings $\ddot{\cal G}_1=\{(\{1\},X),(\{2\},X),(\{3,4,5\},X),(\{7,8,9\},X),(\{6\},X)\}$ and $\ddot{\cal G}_2=\{(\{1,2\},X),(\{3,4,5\},X),(\{7,8,9\},X),(\{6\},X)\}$, as well as the gold standard $\ddot{\cal C}=\{(\{1,\ldots,6\},X),(\{7,8\},X),(\{9\},X)\}$, where $X=\{1',2',3'\}$. These are examples designed to test the compliance of the evaluation measures to the Homogeneity (A.1) condition. Likewise, in order to test the compliance of the evaluation measures to the Rag Bag (A.3) condition, consider the candidate biclusterings $\ddot{\cal G}_1=\{(\{1,\ldots,4\},X),(\{5,\ldots,9\},X)\}$ and $\ddot{\cal G}_2=\{(\{1,\ldots,5\},X),(\{6,\ldots,9\},X)\}$, as well as the gold standard $\ddot{\cal C}=\{(\{1\},X),(\{2\},X),$ $(\{3\},X),(\{4\},X),(\{5\},X),(\{6,\ldots,9\},X)\}$.

For the sake of uniformity in our presentation, in all cases we consider that a biclustering $\ddot{\cal G}_1$ being scored worse than a biclustering $\ddot{\cal G}_2$ by a measure $f$ means that $f(\ddot{\cal G}_1, \ddot{\cal C})<f(\ddot{\cal G}_2, \ddot{\cal C})$, \emph{i.e.} we view the scores yielded by evaluation measures as similarity values. Since CE and RNIA are defined as dissimilarities in the range $[0,1]$, in both cases we transform the scores into similarity values by making $f_{sim}(\ddot{\cal G},\ddot{\cal C})=1-f_{dissim}(\ddot{\cal G},\ddot{\cal C})$.

Table~\ref{tab:comparisonTraditionalCond} shows the scores yielded by CE, RNIA and E4SC for the pairs of traditional clusterings obtained by applying the micro-objects transformation to the biclusterings defined above. Every pair of values represents the scores yielded for $\ddot{\cal G}_1$ and $\ddot{\cal G}_2$, in that order. The emphasized cells highlight cases where the corresponding condition was not satisfied by the corresponding evaluation measure, as $\ddot{\cal G}_1$ was scored equally or better than $\ddot{\cal G}_2$.

\begin{table}[!cht]
\centering
\begin{tabular}{c|c|c}
&Homogeneity
&Rag Bag
\\
\hline
CE&\textbf{0.556~~0.556}&
\textbf{0.556~~0.556}\\ 
\hline
RNIA&\textbf{1.000~~1.000}&
\textbf{1.000~~1.000}\\ 
\hline
E4SC&0.544~~0.606&
\textbf{0.543~~0.533}\\ 
\hline
\end{tabular}
\caption{Compliance, or lack thereof, to conditions A.1 and A.3 by the measures CE, RNIA and E4SC.}
\label{tab:comparisonTraditionalCond}
\end{table}

\section{Our proposal}
\label{sec:ourProposal}

The basis of our new proposal is the measure CICE-BCubed F$_1$, which was presented in \cite{cice,IDApaper} for traditional clusterings, with an emphasis on adequately handling the occurrence of overlapping clusters. There, it was shown that CICE-BCubed F$_1$ satisfies conditions A.1 to A.5. Building on that proposal, our new measure is designed to keep the elements of CICE-BCubed F$_1$ that make it satisfy these conditions, while adapting it to the biclustering scenario in such a way that conditions B.1 to B.5 are also satisfied. We do so by adapting CICE-BCubed F$_1$ to the micro-objects transformation. 

CICE-BCubed F$_1$ is based on BCubed F$_1$ \cite{Bagga}. They both redefine the traditional Information Retrieval measures \emph{Precision} and \emph{Recall}, whose values are combined into F$_1$. The redefinitions introduced by CICE-BCubed and BCubed relate to traditional Precision and Recall in the sense that that they assess the likelihood of decisions made by the clustering algorithm to be correct and the likelihood of known correct decisions to be made by the algorithm, respectively. They differ in the nature of what is considered as a decision. While the original measures treat the action of placing an object in a cluster as a decision, BCubed and CICE-BCubed variants view a decision as the action of making two objects co-occur in a cluster. CICE-BCubed Precision and Recall differ from their BCubed counterparts in the fact that they add an extra term that prevents clusterings that are not identical to the gold standard from being given the maximum score.

In order to evaluate a candidate clustering ${\cal G}=\{G_1,G_2,\ldots,G_{t_1}\}$ with respect to the gold standard ${\cal C}=\{C_1,C_2,\ldots,C_{t_2}\}$, CICE-BCubed Precision computes a score for every pair of objects $o,o' \in O$ as follows

\begin{equation}
\label{eq:varsigma}
\varsigma(o,o')=\frac{\min(|{\cal G}(o)\cap {\cal G}(o')|,|{\cal C}(o)\cap {\cal C}(o')|) \cdot \Phi(o,o')}{|{\cal G}(o)\cap {\cal G}(o')|}
\end{equation}

\noindent
where ${\cal G}(o)=\{G \in {\cal G}:\ o \in G\}$, ${\cal C}(o)=\{C \in {\cal C}:\ o \in C\}$, and the function $\Phi(o,o')$, called \emph{Cluster Identity Index} (CII), averages the degrees of similarity between every candidate cluster containing $o$ and $o'$ and its most similar class, measured through their Jaccard's coefficient, and is defined as

\begin{equation}
\Phi(o,o')=\frac{1}{|{\cal G}(o,o')|} \sum_{G\in {\cal G}(o,o')} {\max_{C \in {\cal C}(o,o')} \left\{\frac{|G \cap C|}{|G \cup C|}\right\}}
\label{eq:CII}
\end{equation}

\noindent
where ${\cal G}(o,o')=\{G \in {\cal G}:\ o \in G \textrm{ and } o' \in G\}$ and ${\cal C}(o,o')=\{C \in {\cal C}:\ o~\in~C \textrm{ and } o' \in C\}$. All pairwise scores are combined into CICE-BCubed Precision as 

\begin{equation}
\label{eq:PrecisionCICE}
Precision({\cal G},{\cal C})=\frac{1}{|O|}\sum_{o\in O}{\frac{1}{| \bigcup_{G\in {\cal G}(o)}{G}|}}\sum_{o'\in E_{\cal G}(o)}{\varsigma(o,o')}
\end{equation}

\noindent
where $E_{\cal G}(o)=\{o':\ o \in G \textrm{ and } o' \in G \textrm{ for some }G \in{\cal G}\}$. Following an analogous strategy, CICE-BCubed Recall computes, for every pair of objects $o,o' \in O$, the score 

\begin{equation}
\label{eq:tau}
\tau(o,o')=\frac{\min(|{\cal G}(o)\cap {\cal G}(o')|,|{\cal C}(o)\cap {\cal C}(o')|) \cdot \Phi(o,o')}{|{\cal C}(o)\cap {\cal C}(o')|}
\end{equation}

\noindent
and combines all pairwise scores as

\begin{equation}
\label{eq:recallCICE}
Recall({\cal G},{\cal C})=\frac{1}{|O|}\sum_{o\in O}{\frac{1}{| \bigcup_{C\in {\cal C}(o)}{C} |}} \sum_{o'\in E_{\cal C}(o)}          {\tau(o,o')}
\end{equation}

\noindent
where $E_{\cal C}(o)=\{o':\ o \in C \textrm{ and } o' \in C \textrm{ for some }C \in{\cal C}\}$. Finally, CICE-BCubed Precision and Recall are combined into CICE-BCubed F$_1$ as

\begin{equation}
F_{1}({\cal G},{\cal C})=\frac{2 \cdot Precision({\cal G},{\cal C}) \cdot Recall({\cal G},{\cal C})}{Precision({\cal G},{\cal C})+Recall({\cal G},{\cal C})}\label{F}
\end{equation}

\vspace{0.5cm}

We now describe our adaptation of CICE-BCubed Precision and Recall for handling the evaluation of biclusterings. As we mentioned earlier, we apply a micro-objects transformation, by which a candidate biclustering $\ddot{\cal G}=\{\ddot{G}_1,\ddot{G}_2,$ $\ldots,\ddot{G}_{t_1}\}$, where $\ddot{G}_i=(\bar{G}_i,\mathring{G}_i)$ for $i \in \{1,\ldots, t_1\}$, is transformed into the candidate clustering $\widetilde{\cal G}=\{\widetilde{G}_1,\widetilde{G}_2,\ldots,\widetilde{G}_{t_1}\}$, where $\widetilde{G}_i=\bar{G}_i \times \mathring{G}_i$ for $i \in \{1,\ldots, t_1\}$. In an analogous manner, the gold standard $\ddot{\cal C}=\{(\bar{C}_1,\mathring{C}_1),(\bar{C}_2,\mathring{C}_2),$ $\ldots,(\bar{C}_{t_2},\mathring{C}_{t_2})\}$ is transformed into $\widetilde{\cal C}=\{\widetilde{C}_1,\widetilde{C}_2,\ldots,\widetilde{C}_{t_2}\}$, with $\widetilde{C}_i=\bar{C}_i \times \mathring{C}_i$ for $i \in \{1,\ldots, t_2\}$.

For $\ddot{\cal G}$ and $\ddot{\cal C}$, we will define our new measure by redefining CICE-BCubed Precision and Recall and combining them into F$_1$. We will refer to these redefinitions as \emph{Micro-object-space-fitted CICE-BCubed}, abbreviated to \emph{MOCICE-BCubed}. First, note that CICE-BCubed Precision and Recall are defined under the assumption that every object in the universe $O$ belongs to at least one candidate clustering and at least one class, \emph{i.e.} $\cup_{G \in {\cal G}}G=\cup_{C \in {\cal C}}C=O$. This assumption is not valid for clusterings obtained from biclusterings by the micro-objects transformation. Instead, we define the (possibly different) sets $$U_{\widetilde{\cal G}}=\bigcup_{i=1}^{t_1}\widetilde{G}_i \textrm{   and   } U_{\widetilde{\cal C}}=\bigcup_{i=1}^{t_2}\widetilde{C}_i$$ which are used for defining MOCICE-BCubed Precision and Recall as

\begin{equation}
\label{eq:PrecisionB2}
Precision(\ddot{\cal G},\ddot{\cal C})=\frac{1}{|U_{\widetilde{\cal G}}|}\sum_{x\in U_{\widetilde{\cal G}}}{\frac{1}{| \bigcup_{\widetilde{G} \in \widetilde{\cal G}(x)}{\widetilde{G}}|}}     \sum_{y\in E_{\widetilde{\cal G}}(x)}{\varsigma(x,y)}
\end{equation}

\noindent
and

\begin{equation}
\label{eq:recallB2}
Recall(\ddot{\cal G},\ddot{\cal C})=\frac{1}{|U_{\widetilde{\cal C}}|} \sum_{x\in U_{\widetilde{\cal C}}}{\frac{1}{|\bigcup_{\widetilde{C}\in \widetilde{\cal C}(x)}\widetilde{C}|}} \sum_{y\in E_{\widetilde{\cal C}}(x)}          {\tau(x,y)}
\end{equation}

Finally, MOCICE-BCubed Precision and Recall are combined into MOCICE-BCubed F$_1$ in the usual manner.

Consider a traditional clustering ${\cal G}=\{G_1,G_2,\ldots,G_{t_1}\}$ and a gold standard ${\cal C}=\{C_1,C_2,\ldots,C_{t_2}\}$. As we mentioned previously, ${\cal G}$ and ${\cal C}$ may be seen as two biclusterings $\ddot{\cal G}=\{(G_1,X),(G_2,X),\ldots,(G_{t_1},X)\}$ and $\ddot{\cal C}=\{(C_1,X),$ $(C_2,X),\ldots,(C_{t_2},X)\}$, where $X \subseteq F$ is an arbitrary feature set. When applied to such biclustering scenario, MOCICE-BCubed F$_1$ is equivalent to CICE-BCubed F$_1$ on the clustering $\widetilde{\cal G}=\{G_1 \times X,G_2 \times X,\ldots,G_{t_1} \times X\}$ and the gold standard $\widetilde{\cal C}=\{C_1 \times X,C_2 \times X,\ldots,C_{t_2} \times X\}$, as $\cup_{G \in {\cal G}}G=O$. Moreover, the following result on CICE-BCubed F$_1$ holds.

\begin{theorem}\label{equivForA1toA5}
Let ${\cal G}=\{G_1,G_2,\ldots,G_{t_1}\}$ and ${\cal C}=\{C_1,C_2,\ldots,C_{t_2}\}$ be a candidate clustering and a gold standard, respectively, on $(O,F)$. Let $X \subseteq F$ be an arbitrary feature set and let $\widetilde{\cal G}=\{G_1 \times X,G_2 \times X,\ldots,G_{t_1} \times X\}$ and $\widetilde{\cal C}=\{C_1 \times X,C_2 \times X,\ldots,C_{t_2} \times X\}$. Then, $$F_1(\widetilde{\cal G},\widetilde{\cal C})=F_1({\cal G},{\cal C}).$$
\end{theorem}

\begin{proof}
It is simple to see that $|\widetilde{\cal G}((o,f))|=|{\cal G}(o)|$ and $|\widetilde{\cal C}((o,f))|=|{\cal C}(o)|$ for every $o \in O$ and every $f \in X$. Moreover, in Eq.~\ref{eq:CII}, we have that for any $G \in {\cal G}$ and any $C \in {\cal C}$, $$\frac{|(G \times X)\cap(C \times X)|}{|(G \times X)\cup(C \times X)|}=\frac{|X|\cdot|G \cap C|}{|X|\cdot|G \cup C|}=\frac{|G \cap C|}{|G \cup C|}.$$

Thus, $\Phi((o,f),(o',f'))=\Phi(o,o')$ for every $o \in O$ and every $f \in X$. In consequence, $\varsigma((o,f),(o',f'))=\varsigma(o,o')$ and $\tau((o,f),(o',f'))=\tau(o,o')$ for every $o \in O$ and every $f \in X$. 

Furthermore, we have that $E_{\widetilde{\cal G}}((o,f))=E_{\cal G}(o) \times X$ and $E_{\widetilde{\cal C}}((o,f))=E_{\cal C}(o) \times X$ for every $o \in O$ and every $f \in X$. Finally, $|\bigcup_{\widetilde{G} \in \widetilde{\cal G}((o,f))}\widetilde{G}|=|X|\cdot|\bigcup_{G \in {\cal G}(o)}G|$, so

\begin{displaymath}
\begin{array}{rcl}
Precision(\widetilde{\cal G},\widetilde{\cal C})&=&\displaystyle\frac{1}{|O|\cdot|X|}\displaystyle\sum_{(o,f)\in O \times X}{\frac{1}{\left|\displaystyle\bigcup_{\widetilde{G}\in \widetilde{\cal G}((o,f))}{\widetilde{G}}\right|}}\displaystyle\sum_{(o',f')\in E_{\widetilde{\cal G}}((o,f))}{\varsigma((o,f),(o',f'))}\\
\\
&=&\displaystyle\frac{1}{|O|\cdot|X|}\displaystyle\sum_{(o,f)\in O \times X}{\frac{|X|}{|X|\cdot\left|\displaystyle\bigcup_{G\in {\cal G}(o)}{G}\right|}}\sum_{o' \in E_{\cal G}(o)}{\varsigma(o,o')}\\
\\
&=&\displaystyle\frac{|X|}{|O|\cdot|X|}\sum_{o \in O}{\frac{1}{|\bigcup_{G\in {\cal G}(o)}{G}|}}\sum_{o' \in E_{\cal G}(o)}{\varsigma(o,o')}\\
\\
&=&Precision({\cal G},{\cal C})
\end{array}
\end{displaymath}

\noindent
and, by an analogous reasoning, $Recall(\widetilde{\cal G},\widetilde{\cal C})=Recall({\cal G},{\cal C})$. In consequence, F$_1(\widetilde{\cal G},\widetilde{\cal C})=\textrm{F}_1({\cal G},{\cal C})$, so the proof is complete.
\end{proof}

As a consequence of Theorem~\ref{equivForA1toA5}, we can say that transforming a clustering ${\cal G}$ and a gold standard ${\cal C}$ into biclusterings and computing MOCICE-BCubed F$_1$ on them is equivalent to computing CICE-BCubed F$_1$ on ${\cal G}$ and ${\cal C}$, so under this transformation conditions A.1 to A.5 continue to be satisfied. In what follows, we will show that MOCICE-BCubed F$_1$ also satisfies conditions B.1 to B.5.

\begin{theorem}\label{th:PNIA}
MOCICE-BCubed F$_1$ satisfies condition B.1 (Penalty for non-intersection area).
\end{theorem}

\begin{proof}
Let $\ddot{\cal G}$ and $\ddot{\cal G}'$ be a pair of candidate biclusterings satisfying the premises of condition B.1, and let $\ddot{\cal C}$ be the gold standard. We have that $Precision(\ddot{\cal G},\ddot{\cal C})<Precision(\ddot{\cal G}',\ddot{\cal C})$ because of the extra, incorrectly clustered, object(s) involved in $\ddot{\cal G}'$. Moreover, $Recall(\ddot{\cal G},\ddot{\cal C})=Recall(\ddot{\cal G}',\ddot{\cal C})$, so $F_1(\ddot{\cal G},\ddot{\cal C})<F_1(\ddot{\cal G}',\ddot{\cal C})$, as required by the condition.
\end{proof}

\begin{theorem}\label{th:background}
MOCICE-BCubed F$_1$ satisfies condition B.2 (Background independence).
\end{theorem}

\begin{proof}
The result is a direct consequence of the manner in which $U_{\widetilde{\cal G}}$ and $U_{\widetilde{\cal C}}$ are defined.
\end{proof}

\begin{theorem}\label{th:scale_inv}
MOCICE-BCubed F$_1$ satisfies condition B.3 (Scale invariance).
\end{theorem}

\begin{proof}
Let $\ddot{\cal G}$ be a candidate biclustering and let $\ddot{\cal C}$ be a gold standard. Let $\ddot{\cal G}'$ and $\ddot{\cal C}'$ be the $k$-scaled versions of $\ddot{\cal G}$ and $\ddot{\cal C}$, respectively, and let $\widetilde{\cal G}$, $\widetilde{\cal G}'$, $\widetilde{\cal C}$ and $\widetilde{\cal C}'$ be the micro-object-space clusterings into which $\ddot{\cal G}$, $\ddot{\cal G}'$, $\ddot{\cal C}$ and $\ddot{\cal C}'$ are transformed, respectively. It is simple to see that $|\widetilde{\cal G}'(x)|=|\widetilde{\cal G}(x)|$ and $|\widetilde{\cal C}'(x)|=|\widetilde{\cal C}(x)|$ for any $x$. Moreover, for any $\widetilde{G}\in\widetilde{\cal G}$ and any $\widetilde{C} \in \widetilde{\cal C}$, $$\frac{|\widetilde{G}' \cap\widetilde{C}'|}{|\widetilde{G}'\cup\widetilde{C}'|} = \frac{k\cdot|\widetilde{G} \cap\widetilde{C}|}{k\cdot|\widetilde{G}\cup\widetilde{C}|}= \frac{|\widetilde{G} \cap\widetilde{C}|}{|\widetilde{G}\cup\widetilde{C}|},$$

\noindent
so the value yielded by $\widetilde{\Phi}(x,y)$ for any pair $x,y$  when evaluating $(\ddot{\cal G}',\ddot{\cal C}')$ is the same yielded when evaluating $(\ddot{\cal G},\ddot{\cal C})$. As a consequence of the aforementioned facts, we have that the values yielded by $\varsigma(x,y)$ and $\tau(x,y)$ for any pair $x,y$  when evaluating $(\ddot{\cal G}',\ddot{\cal C}')$ are also the same yielded when evaluating $(\ddot{\cal G},\ddot{\cal C})$. Thus,

\begin{displaymath}
\begin{array}{rcl}
Precision(\ddot{\cal G}',\ddot{\cal C}')&=&\displaystyle\frac{1}{|U_{\widetilde{\cal G}'}|}\sum_{x\in U_{\widetilde{\cal G}'}}{\frac{1}{| \bigcup_{\widetilde{G}' \in \widetilde{\cal G}'(x)}{\widetilde{G}'}|}}\sum_{y\in E_{\widetilde{\cal G}'}(x)}{\varsigma(x,y)}\\
&=&\displaystyle\frac{1}{k\cdot|U_{\widetilde{\cal G}}|}\sum_{x\in U_{\widetilde{\cal G}}}k\cdot\frac{1}{k\cdot|\bigcup_{\widetilde{G} \in \widetilde{\cal G}(x)}{\widetilde{G}}|}\sum_{y\in E_{\widetilde{\cal G}}(x)}{k\cdot\varsigma(x,y)}\\
&=&Precision(\ddot{\cal G},\ddot{\cal C})
\end{array}
\end{displaymath}

\noindent
and, by an analogous reasoning, $Recall(\ddot{\cal G}',\ddot{\cal C}')=Recall(\ddot{\cal G},\ddot{\cal C})$. In consequence, $F_1(\ddot{\cal G}',\ddot{\cal C}')=F_1(\ddot{\cal G},\ddot{\cal C})$, so the proof is complete.
\end{proof}

\begin{theorem}\label{th:copy_inv}
MOCICE-BCubed F$_1$ satisfies condition B.4 (Copy invariance).
\end{theorem}

\begin{proof}
Let $\ddot{\cal G}$ be a candidate biclustering and let $\ddot{\cal C}$ be a gold standard. Let $\ddot{\cal G}'$ and $\ddot{\cal C}'$ be the $k$-copied versions of $\ddot{\cal G}$ and $\ddot{\cal C}$, respectively, and let $\widetilde{\cal G}$, $\widetilde{\cal G}'$, $\widetilde{\cal C}$ and $\widetilde{\cal C}'$ be the micro-object-space clusterings into which $\ddot{\cal G}$, $\ddot{\cal G}'$, $\ddot{\cal C}$ and $\ddot{\cal C}'$ are transformed, respectively. For any $x$, we have that $\widetilde{\cal G}'(x)$ is the disjoint union of $k$ copies of $\widetilde{\cal G}(x)$ and $\widetilde{\cal C}'(x)$ is the disjoint union of $k$ copies of $\widetilde{\cal C}(x)$. Moreover,

\begin{displaymath}
\begin{array}{c}
\displaystyle\frac{1}{|\widetilde{\cal G}'(x,y)|}\sum_{\widetilde{G}'\in \widetilde{\cal G}'(x,y)} {\max_{\widetilde{C}' \in \widetilde{\cal C}'(x,y)} \left\{\frac{|\widetilde{G}' \cap \widetilde{C}'|}{|\widetilde{G}' \cup \widetilde{C}'|}\right\}}=\\
=\displaystyle\frac{1}{k\cdot|\widetilde{\cal G}(x,y)|}\sum_{\widetilde{G}\in \widetilde{\cal G}(x,y)} {k\cdot\max_{\widetilde{C} \in \widetilde{\cal C}(x,y)} \left\{\frac{|\widetilde{G} \cap \widetilde{C}|}{|\widetilde{G} \cup \widetilde{C}|}\right\}}=\\
=\displaystyle\frac{1}{|\widetilde{\cal G}(x,y)|}\sum_{\widetilde{G}\in \widetilde{\cal G}(x,y)} {\max_{\widetilde{C} \in \widetilde{\cal C}(x,y)} \left\{\frac{|\widetilde{G} \cap \widetilde{C}|}{|\widetilde{G} \cup \widetilde{C}|}\right\}},
\end{array}
\end{displaymath}

\noindent
so the value yielded by $\widetilde{\Phi}(x,y)$ for any pair $x,y$  when evaluating $(\ddot{\cal G}',\ddot{\cal C}')$ is the same yielded when evaluating $(\ddot{\cal G},\ddot{\cal C})$. As a consequence of the aforementioned facts, we have that the values yielded by $\varsigma(x,y)$ and $\tau(x,y)$ for any pair $x,y$  when evaluating $(\ddot{\cal G}',\ddot{\cal C}')$ are also the same yielded when evaluating $(\ddot{\cal G},\ddot{\cal C})$, as the role of the $k$ copies is simplified out when performing the divisions in Eqs.~\ref{eq:varsigma} and~\ref{eq:tau}. Thus,

\begin{displaymath}
\begin{array}{rcl}
Precision(\ddot{\cal G}',\ddot{\cal C}')&=&\displaystyle\frac{1}{|U_{\widetilde{\cal G}'}|}\sum_{x\in U_{\widetilde{\cal G}'}}{\frac{1}{| \bigcup_{\widetilde{G}' \in \widetilde{\cal G}'(x)}{\widetilde{G}'}|}}\sum_{y\in E_{\widetilde{\cal G}'}(x)}{\varsigma(x,y)}\\
&=&\displaystyle\frac{1}{|U_{\widetilde{\cal G}}|}\sum_{x\in U_{\widetilde{\cal G}}}\frac{1}{|\bigcup_{\widetilde{G} \in \widetilde{\cal G}(x)}{\widetilde{G}}|}\sum_{y\in E_{\widetilde{\cal G}}(x)}{\varsigma(x,y)}\\
&=&Precision(\ddot{\cal G},\ddot{\cal C})
\end{array}
\end{displaymath}

\noindent
and, by an analogous reasoning, $Recall(\ddot{\cal G}',\ddot{\cal C}')=Recall(\ddot{\cal G},\ddot{\cal C})$. In consequence, $F_1(\ddot{\cal G}',\ddot{\cal C}')=F_1(\ddot{\cal G},\ddot{\cal C})$, so the proof is complete.
\end{proof}

\begin{theorem}\label{th:multilple_cluster}
MOCICE-BCubed F$_1$ satisfies condition B.5 (Multiple cluster coverage penalty).
\end{theorem}

\begin{proof}
The result is a direct consequence of MOCICE-BCubed F$_1$ being equivalent to CICE-BCubed F$_1$ on the micro-objects space and CICE-BCubed F$_1$ satisfying condition A.5 (Perfect match) which guarantees that the optimum score is given to a candidate clustering if and only if it is identical to the gold standard.
\end{proof}

Summing up, Theorems~\ref{th:PNIA} to~\ref{th:multilple_cluster} show that MOCICE-BCubed F$_1$ satisfies conditions B.1 to B.5.

\section{Conclusions}
\label{sec:conclusion}

In this paper we have presented MOCICE-BCubed F$_1$, a new external evaluation measure for biclustering algorithms. This measure is an adaptation, based on the micro-objects transformation, of CICE-BCubed F$_1$, which had been previously shown to satisfy the most comprehensive set of meta-evaluation conditions for the traditional clustering task. We show that the new measure is equivalent to CICE-BCubed F$_1$ when evaluating traditional clustering, viewed as a particular case of biclustering, thus inheriting the compliance to the most comprehensive set of meta-evaluation conditions for this task. This behaviour sets MOCICE-BCubed F$_1$ apart from previously proposed micro-object-based measures, for which we provide counterexamples showing lack of compliance with several of these conditions. Moreover, we show that MOCICE-BCubed F$_1$ also satisfies the most comprehensive set of meta-evaluation conditions specific to the biclustering task.

Our main direction for future work has to do with the practical difficulty of hand-annotating gold standard collections for biclustering evaluation with both the object clusters and the associated feature subspaces. To that end, we intend to propose measures capable of using traditional clustering gold standards, of which there is a much larger availability, to evaluate biclusterings.

\end{document}